%% file: aaai2027.tex
\documentclass[letterpaper]{article} 
\usepackage{aaai2027}  
\usepackage[hyphens]{url}  
\usepackage{graphicx} 
\usepackage{natbib}  
\usepackage{caption} 
\usepackage{algorithm}
\usepackage{algorithmic}
\newcommand{\ourmethod}{{HERO}}
\usepackage{newfloat}
\usepackage{listings}
\DeclareCaptionStyle{ruled}{labelfont=normalfont,labelsep=colon,strut=off} 
\floatstyle{ruled}
\newfloat{listing}{tb}{lst}{}
\floatname{listing}{Listing}

\usepackage{booktabs}
\usepackage{amsmath,amssymb,amsfonts}%
\usepackage{amsthm}%
\usepackage{mathrsfs}%
\newtheorem{definition}{Definition}%
\theoremstyle{plain}%
\newtheorem{theorem}{Theorem}
\newtheorem{lemma}[theorem]{Lemma}%
\newtheorem{corollary}[theorem]{Corollary}%

\theoremstyle{definition}%
\newtheorem{remark}{Remark}%

\theoremstyle{remark}%

\usepackage{xcolor}

\title{Overcoming the Weakest-Link Effect in LLM-Driven Program Optimization via Heterogeneous Edit Recombination}
\author{
    Jingwen Fu\equalcontrib\textsuperscript{\rm 1,\rm 2},
    Zhen Liu\equalcontrib\textsuperscript{\rm 3},
    Yuhan Liu\textsuperscript{\rm 4},\\
    He Zhang\textsuperscript{\rm 1,\rm 2}\corresponding,
    Nanning Zheng\textsuperscript{\rm 3}\corresponding
}
\affiliations{
    \textsuperscript{\rm 1}Zhongguancun Academy\\
    \textsuperscript{\rm 2}Zhongguancun Institute of Artificial Intelligence\\
    \textsuperscript{\rm 3}National Key Laboratory of Human-Machine Hybrid Augmented Intelligence, and Institute of Artificial Intelligence and Robotics, Xi'an Jiaotong University\\
    \textsuperscript{\rm 4}Xiaomi Inc\\
    jwfu99@gmail.com, zhanghe@zgci.ac.cn, nnzheng@mail.xjtu.edu.cn
}

\begin{document}

\maketitle

\begin{abstract}

Large language models (LLMs) are increasingly used to solve complex problems by searching over program space, offering a general paradigm for scientific problems that can be naturally represented and solved as programs. Despite recent progress, identifying effective optimization directions for a candidate program remains challenging. By analogy with automatic differentiation, existing methods typically guide the search using a textual ``gradient'': a first-order update direction expressed as textual edits. Such gradients are inferred either from previously evaluated programs or from LLM-generated feedback on the implicit program-score mapping. However, these estimates become increasingly unreliable as the program--score mapping grows more complex, limiting their practical utility.
We argue that explicit gradients are not essential for effective program optimization. Leveraging their prior knowledge, LLMs can propose plausible atomic edits directly from the current program, thereby enabling a zeroth-order optimization strategy. However, zeroth-order search suffers from a \textit{weakest-link effect}: when a bundle of edits is accepted or rejected as a whole, a single harmful edit can negate the benefits of all remaining edits. To address this issue, we introduce \ourmethod, a program optimizer that prompts an LLM to generate diverse, non-overlapping atomic edits and then systematically selects and composes them into coherent program improvements using evaluator scores. We evaluate \ourmethod~across algorithmic problems, strategy games, the design of LLM-based agentic systems, and robotic path planning. Across these domains, \ourmethod~consistently discovers higher-scoring programs and converges substantially faster than prior LLM-based optimizers, while consuming fewer tokens.
\end{abstract}


\section{Introduction}


Optimization over program space provides a general framework for solving complex problems, ranging from mathematical construction to agent design. Such problems can often be expressed naturally as programs, which are executable, automatically evaluable, and compositional~\citep{zhang2024revolve,yuksekgonul2025optimizing,ding2025scaling}. LLMs are well suited to this setting as they can reason about program semantics, draw on broad prior knowledge, and iteratively revise candidate solutions~\citep{zhou2025riot,madaan2023self,zhang2025codegrad,FunSearchNature2024}. The central difficulty, however, is how to steer the search. At each iteration, the optimizer must determine an \emph{optimization direction} to guide how the current program should be modified to improve its performance~\citep{nie2023importance,liu2026structured,liu2026essence,liu2026instruction}. 
Existing methods typically infer this direction from the observed performance of previously evaluated candidates, which we refer to as \emph{first-order information}~\citep{yuksekgonul2025optimizing,cheng2024trace}. The underlying intuition is that the evaluation history provides a local description of the objective: by identifying which past modifications helped or hurt, the optimizer can infer a useful update direction, analogous to estimating a numerical gradient. Consequently, access to evaluation scores is often treated as essential.

\begin{figure}[t]
    \centering    \includegraphics[width=0.8\linewidth]{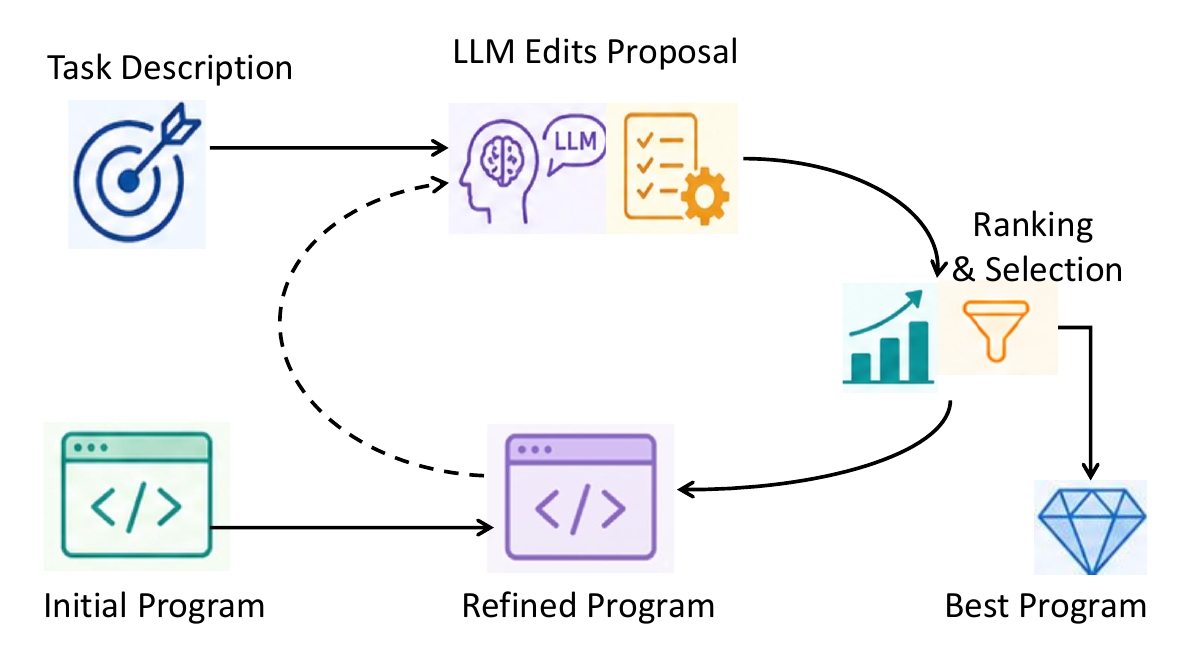}
    \caption{Workflow of \ourmethod.
    Given a task description, \ourmethod{} iteratively improves the program by applying selectively recombined edits proposed by the LLM.
    }
    \label{fig:paradigm}
    \vspace{-0.2in}
\end{figure}

\begin{figure*}
    \centering
    \includegraphics[width=0.9\textwidth]{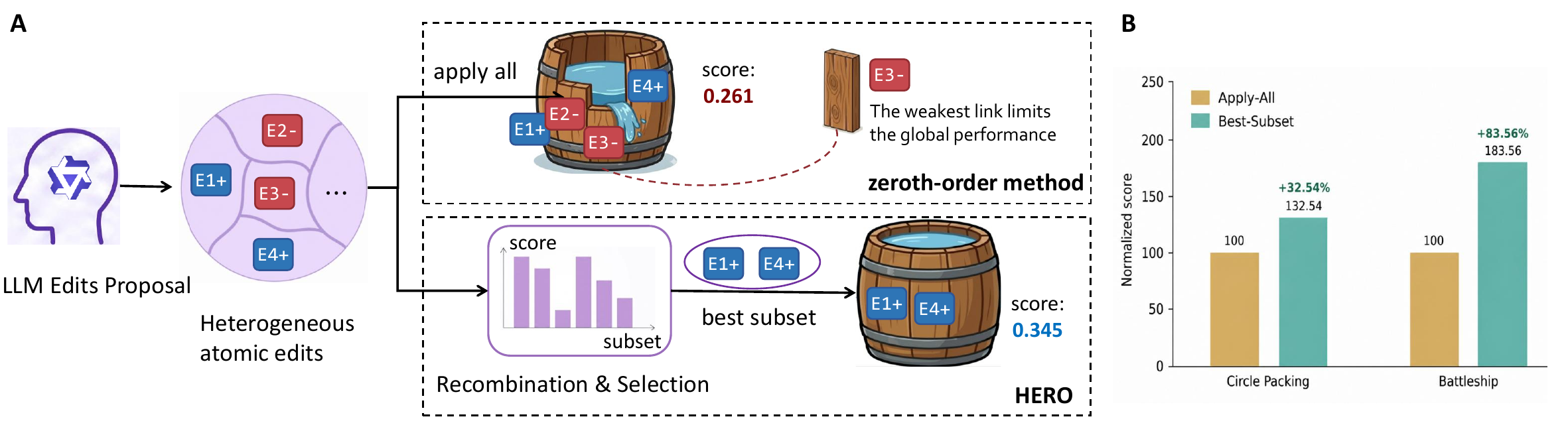}
    \caption{Overview of \ourmethod.
    \textbf{(A)} Working mechanism. At each step, the LLM generates multiple non-overlapping, heterogeneous atomic edits. Exsiting methods apply or reject all edits by the global score, hence suffering from the weakest-link effect. \ourmethod{} instead recombines the edits and selects the subset that maximizes performance, neutralizing the influence of harmful edits.
    \textbf{(B)} Best-subset recombination mitigates the weakest-link effect and leads to substantial performance improvement.
    }
    \label{fig:framework}
    \vspace{-0.1in}
\end{figure*}


We posit that this assumption becomes unreliable as the program--score mapping grows more complex. 
Inferring an update from performance can be regarded as a credit-assignment problem: the optimizer must attribute a scalar score to the many interacting components of a program and determine which components should be changed. 
When the program--score mapping is simple, this attribution could be tractable. 
Nevetheless, for those open scientific problems remained underexplored by LLMs,
the candidate programs typically involve long execution traces, tightly coupled components, or stochastic outcomes. 
As a result, the same evaluator score can be consistent with several different explanations, making the inferred attribution or `gradient direction' fundamentally ambiguous. 
We theoretically formulate this limitation in Appendix A, and show that a finite history of scores cannot uniquely determine the correct update direction when the decoder from scores to update labels is unconstrained.

The limitations of score-based reasoning raise the central question of this paper: \textit{Can an LLM-based optimizer remain effective without reasoning about ground-truth performance?} 
Our starting point is a simple empirical evidence observed in prior research~\citep{yang2024swe,zhang2024autocoderover,xia2024agentless}: even without performance feedback, LLMs possess substantial implicit knowledge of how a candidate solution might be improved (Fig.~\ref{fig:paradigm}). Given only a task description and a candidate program, they can often propose plausible, localized \emph{atomic edits}. 
The main difficulty, we argue, lies not in generating individual edits but in combining them. 
As shown in Fig.~\ref{fig:framework}A, a straightforward approach is to apply all proposed edits together and accept or reject the resulting program according to its global score. 
However, a batch may contain both beneficial and harmful edits since edit generation is stochastic and unguided by performance feedback. 
Under this all-or-nothing decision rule, a single harmful edit can outweigh the gains from the others, causing the entire batch, including its useful edits, to be discarded. We refer to this failure mode as the \textbf{weakest-link effect}. Simply instructing the model to generate only beneficial edits does not resolve the problem, because the model cannot reliably determine the effect of each edit before evaluation without ground-truth feedback. Harmful edits must therefore be filtered during selection, by evaluating modifications at a finer granularity than the full batch.


Motivated by this observation, we introduce \textbf{\ourmethod} (\textbf{H}eterogeneous \textbf{E}dit \textbf{R}ecombination \textbf{O}ptimizer), which separates edit generation from edit selection (Fig.~\ref{fig:framework}A). During generation, the LLM proposes optimization directions using its internal knowledge; during selection, scalar evaluations identify effective combinations of those directions. At each iteration, \ourmethod{} provides the LLM with only the task description and the current program, deliberately withholding its ground-truth performance, and asks it to generate a set of \emph{heterogeneous, non-overlapping atomic edits}. Instead of applying the entire set, \ourmethod{} evaluates recombinations of these edits and retains the best-performing subset, preventing a single harmful edit from overriding beneficial ones. 
Ground-truth performance is used only as a scalar criterion for comparing edit subsets and evaluating the resulting program; it is never interpreted by the LLM or used to infer an update direction. In this sense, \ourmethod{} performs zeroth-order optimization while converting the LLM's latent knowledge of plausible edits into reliable program improvements. We evaluate \ourmethod{} in four domains: Circle Packing~\citep{novikov2025alphaevolve}, Strategy Games~\citep{dor2018strategy}, Agent System Design~\citep{zhangaflow}, and Path Planning~\citep{meng2024llm,xiao2023llm}. Together, these domains cover combinatorial construction, adversarial and partially observable decision making, agentic reasoning, and planning.


Overall, this paper makes three contributions:
\begin{itemize}
    \item We identify a key bottleneck in LLM-based program optimization: the challenge is not generating promising update directions, but composing heterogeneous edits without suffering from the \emph{weakest-link effect} induced by all-or-nothing acceptance.
    \item We introduce \ourmethod, which separates edit generation from edit selection. It generates heterogeneous, non-overlapping atomic edits and selectively recombines them using scalar evaluations, without requiring the LLM to reason about ground-truth performance.
    \item Across four domains, \ourmethod{} achieves faster convergence, lower token consumption, and higher final performance. 
\end{itemize}

Finally, our results are not a claim that first-order information is useless. By charting what is achievable \emph{without} it, we aim to recalibrate the community's estimate of its value and to inform when, and how, it is worth the cost of methods that \emph{do} leverage it. The full discussion of related work is deferred to Appendix D.

\section{Method}

\subsection{Formulation and Intuition}
\label{subsec:formulation}

\paragraph{Problem setup.}
Let $\mathcal{X}$ be the solution space and $f:\mathcal{X}\to\mathbb{R}^{+}$ a target (performance) function. The optimizer seeks a maximizer
\begin{equation}
\label{eq:objective}
x^{\star}\in\operatorname*{arg\,max}_{x\in\mathcal{X}} f(x),
\end{equation}
and proceeds iteratively: at step $t$ it maintains an incumbent $x_t\in\mathcal{X}$ together with an \emph{evaluation history}
\[
H_t \;=\; \bigl\{(x_i,\,f(x_i))\bigr\}_{i=1}^{t}.
\]
A single step proposes a \emph{modification} $\alpha:\mathcal{X}\to\mathcal{X}$ and moves to $\alpha(x_t)$. To contrast \ourmethod{} with prior work precisely, we first classify the information available to such a proposal, then explain why \ourmethod{} deliberately uses only part of it.

\paragraph{First- vs.\ zeroth-order information.}
We distinguish two sources of information a proposal mechanism may exploit. \emph{First-order information} is any functional of the evaluation history $H_t$ used to estimate the local ascent direction $\nabla f(x_t)$~\citep{cheng2024trace, zhanghessiangrad}. The aim is not to reconstruct the entire gradient field but only the direction at the incumbent $x_t$; this is presumed feasible because $H_t$ implicitly constrains $f$. The difficulty is that recovering $\nabla f(x_t)$ from $H_t$ is an \emph{inverse} problem, and Appendix shows it is ill-posed once $f$ is complex: a finite history is consistent with many distinct directions, so the inferred direction is unreliable precisely in the hard regime. \emph{Zeroth-order information}, by contrast, is any information independent of $H_t$, namely that carried by the task description and the model's internal prior. Formally, a first-order optimizer samples a modification from a distribution conditioned on the full context,
\[
\alpha\sim \pi\bigl(\cdot \mid \mathrm{Task},\,H_t,\,x_t\bigr),
\]
drawing on both sources, whereas a purely zeroth-order optimizer conditions only on $(\mathrm{Task},\,x_t)$.

\paragraph{A latent-objective model of zeroth-order proposals.}
As a zeroth-order optimizer, \ourmethod{} samples $\alpha\sim\pi\bigl(\cdot\mid \mathrm{Task},\,x_t\bigr)$
withholding all scores, so the proposed direction cannot be inferred from $H_t$ and must originate from the model's prior. To analyze this prior we adopt the following modeling assumption---a tractable abstraction, not a claim about internal mechanics. We posit a finite set of \emph{latent objectives} $g_1,\dots,g_h:\mathcal{X}\to\mathbb{R}$, each encoding a property the model associates with quality (e.g., ``avoid boundary waste'' in packing, or ``protect corners'' in Othello). We model each sampled modification $\alpha$ as an approximate ascent step on a \emph{random mixture} of these objectives:
\begin{equation}
\label{eq:mixture}
g^{\alpha}_{\theta}=\sum_{i=1}^{h}\theta^{\alpha}_i\, g_i,
\qquad
\alpha(x_t)-x_t \;\approx\; \eta\,\nabla g^{\alpha}_{\theta}(x_t),
\end{equation}
where $\eta>0$ is a step size and the weight vector $\theta^{\alpha}=(\theta^{\alpha}_1,\dots,\theta^{\alpha}_h)$ is drawn afresh for each sample, reflecting the stochasticity of generation. Under this model a single edit is a noisy, partially correct direction: it ascends \emph{some} objective the model believes in, but not necessarily $f$.

\paragraph{Objective: aligning the aggregate with $f$.}
Since the mixture weights are random, no single $\alpha$ is guaranteed to increase $f$. The problem \ourmethod{} solves is therefore one of \emph{aggregation}: select a sub-collection $\{\alpha_1,\dots,\alpha_n\}$ whose combined direction is positively aligned with the true objective near $x_t$, i.e., there exists $c>0$ such that
\begin{equation}
\label{eq:align}
\nabla\!\sum_{i=1}^{n} g^{\alpha_i}_{\theta}(x_t)\;\approx\; c\,\nabla f(x_t).
\end{equation}
Equation~\eqref{eq:align} formalizes the intuition that combining many partially correct edits cancels their objective-specific (off-$f$) components while reinforcing the shared component that tracks $f$, converting diffuse prior knowledge into a usable ascent direction without ever reasoning over scores.

\paragraph{The weakest-link effect.}
A naive aggregation applies \emph{all} proposed edits and accepts the result by its global score; we show this squanders the very alignment that Eq.~\eqref{eq:align} targets. Let $\{\alpha_1,\dots,\alpha_n\}$ be the proposals, write the joint update $\alpha_{1:n}=\alpha_1\circ\cdots\circ\alpha_n$, and let $\delta_i \;=\; f\bigl(\alpha_i(x_t)\bigr)-f(x_t)$
denote the marginal gain of edit $i$. When the edits act on disjoint parts of the solution---as enforced in Section~\ref{subsec:GSFramework}---their effects are approximately additive,
\begin{equation}
\label{eq:additive}
f\bigl(\alpha_{1:n}(x_t)\bigr)-f(x_t)\;\approx\;\sum_{i=1}^{n}\delta_i .
\end{equation}
\emph{All-or-nothing} acceptance keeps the bundle iff this sum is positive. Because the model receives no feedback, it cannot certify any edit, so some $\delta_j<0$ is essentially unavoidable; by Eq.~\eqref{eq:additive}, a single sufficiently harmful edit---one with $\delta_j<-\sum_{i\neq j}\delta_i$ when the remaining edits are beneficial---renders the total non-positive and triggers rejection of the \emph{entire} bundle, discarding every beneficial edit along with the harmful one. We call this the \textbf{weakest-link effect}: the fate of the bundle is dictated by its worst component. It explains the puzzle from Section~1---the model proposes sound edits, yet all-or-nothing acceptance prevents them from accumulating into gains. The remedy is to evaluate at a finer granularity than the full bundle: the subset $S^{\star}=\operatorname*{arg\,max}_{S\subseteq[n]}\sum_{i\in S}\delta_i$ retains exactly the beneficial edits $\{i:\delta_i>0\}$ and attains gain $\sum_{i:\delta_i>0}\delta_i\ge\sum_{i=1}^{n}\delta_i$. This subset search is the core of \ourmethod{}, developed next.

\subsection{Generation--Selection Framework}
\label{subsec:GSFramework}

The analysis of Section~\ref{subsec:formulation} isolates two distinct concerns: \emph{proposing} candidate directions, which the model's prior already supports, and \emph{verifying} which of them to keep, which all-or-nothing acceptance handles poorly. \ourmethod{} decouples these into a knowledge-driven \emph{Generation} phase and a feedback-driven \emph{Selection} phase. Starting from an initial program $x_{0}\in\mathcal{X}$, it applies the composite update
\begin{equation}
\label{eq:update}
\begin{aligned}
    & x_{t} = \mathrm{Update}(x_{t-1}) = \bigl(\mathrm{Select}\circ\mathrm{Generate}\bigr)(x_{t-1}),
\end{aligned}
\end{equation}
and returns $x_{T}$ as the final solution. Generation supplies candidate edits from the model's prior alone, while selection uses ground-truth performance purely as a scalar criterion to retain the edits that help. We specify each phase in turn.

\paragraph{Generation phase.}
Queried with only the task description and the incumbent $x_{t-1}$---no scores---the model returns a set of atomic edits
\[
    A=\{\alpha_1,\dots,\alpha_n\},\qquad \alpha_i:\mathcal{X}\to\mathcal{X},
\]
whose cardinality $n$ is chosen by the model itself, so that granularity adapts to how much the program admits improvement. We require the edits to be \emph{heterogeneous} and pairwise \emph{commuting}: for all $i,j\in[n]$,
\begin{equation}
\label{eq:commute}
    \alpha_i \circ \alpha_j = \alpha_j \circ \alpha_i .
\end{equation}
Commutativity is what makes subset selection well posed. For any index subset $S\subseteq[n]$, define the recombined program
\[
    A_S(x)=\Bigl(\,\bigcirc_{i\in S}\alpha_i\,\Bigr)(x),
\]
the composition of the edits indexed by $S$. By Eq.~\eqref{eq:commute} this composition is independent of the order of application, so $A_S$ is unambiguous and an instruction such as ``keep edits $\{i,k\}$ but drop $j$'' is well defined; without commutativity the effect of a subset would depend on application order and on which other edits are present, and recombination would be ill posed. Commutativity thus turns recombination into a clean combinatorial choice over the $2^{n}$ subsets of disjoint, heterogeneous edits.

\paragraph{Selection phase.}
Given $A$, selection returns the best-performing subset at the incumbent $x_{t-1}$:
\begin{equation}
\label{eq:selectopm}
\begin{gathered}
S^{\star}\in\operatorname*{arg\,max}_{S \subseteq [n]} 
\; \widehat{D}_{S}\, f(x_{t-1}), \\
\widehat{D}_{S}\, f(x_{t-1})
=
\frac{f\!\left(A_S(x_{t-1})\right)-f(x_{t-1})}{\lVert S \rVert}.
\end{gathered}
\end{equation}
after which we set $x_t=A_{S^{\star}}(x_{t-1})$. Here $\widehat{D}_{S}\, f(x_{t-1})$ is the empirical directional improvement of $f$ along the recombined edit (the finite-sample analogue of the directional derivative $D_S f$), $A_S$ is as defined above, and $\lVert S \rVert$ is the update magnitude, set to $1$ throughout for simplicity. Two observations connect this rule to Section~\ref{subsec:formulation}. First, because the optimum ranges over \emph{all} subsets rather than the single full bundle $S=[n]$, harmful edits are excluded while beneficial ones are retained: under the additive approximation of Eq.~\eqref{eq:additive} with $\lVert S \rVert=1$, Eq.~\eqref{eq:selectopm} reduces to $\max_{S\subseteq[n]}\sum_{i\in S}\delta_i$, whose maximizer is exactly $\{i:\delta_i>0\}$. This is the precise mechanism by which \ourmethod{} neutralizes the weakest-link effect. Second, Eq.~\eqref{eq:selectopm} uses $f$ only as a black-box scalar to \emph{compare} candidates; it never asks \emph{why} one subset scores higher, keeping \ourmethod{} strictly zeroth-order in the sense of Section~\ref{subsec:formulation}.

\subsection{Design Details}
\label{sec:design_details}

HERO follows a score-free generation and evaluator-based selection procedure.
Given the task description and current program, the LLM produces localized,
pairwise non-overlapping \texttt{SEARCH}/\texttt{REPLACE} edits
\citep{novikov2025alphaevolve}. The non-overlapping constraint makes edit
composition order-independent and enables well-defined subset recombination.
Selection first removes edits that cause execution errors, then evaluates
candidate subsets and retains the highest-scoring recombination. We enumerate
all subsets when feasible; otherwise, we use a budgeted subset pool. For
expensive evaluator, we adopt an approximately order-preserving surrogate. The
evaluation budget therefore balances subset coverage against ranking
reliability. Evaluator scores are used only for subset comparison and are never
provided to the LLM. See more implementation details in
Appendix E.

\paragraph{Discussion. }
We close by stating precisely the role ground-truth performance plays in \ourmethod{}, since this is exactly what separates it from first-order methods and substantiates the zeroth-order claim made in the abstract and introduction. During optimization the signal enters at a single point---the comparison of edit subsets in Eq.~\eqref{eq:selectopm}---and at one further point afterward, the evaluation of the returned program $x_T$. It is never \emph{reasoned about}: \ourmethod{} extracts no explanation of \emph{why} one subset outscores another and infers no structural insight from the score. All directional knowledge originates from the model's prior (the Generation phase), whereas performance acts solely as a scalar comparison oracle (the Selection phase). This cleanly classifies \ourmethod{} as a \emph{zeroth-order} optimizer in the sense of Section~\ref{subsec:formulation}, and it explains why \ourmethod{} sidesteps the credit-assignment failure formalized in Appendix A: by \emph{ranking} candidates rather than inverting a scalar score back onto a solution's parts, it never attempts the ill-posed inversion that renders first-order reasoning unreliable in the hard regime.

\section{Experiments}
\begin{figure*}
    \centering
    \includegraphics[width=0.99\linewidth]{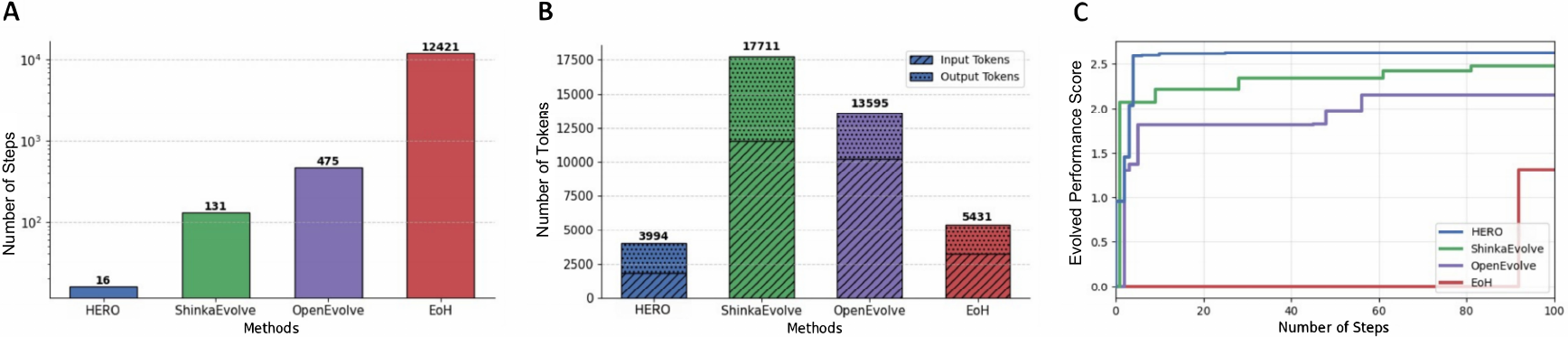}
    \caption{Results on the Circle Packing task. \textbf{A.} Number of steps required to reach a configuration with $\texttt{sum\_radii} > 2.60$. \textbf{B.} Average number of tokens consumed per evolution step. \textbf{C.} Optimization trajectories of different methods during the search.}
    \label{fig:circlepacking}
    \vspace{-0.05in}
\end{figure*}
Our experiments are designed to test the paper's central claim, not merely to report wins. If the bottleneck in LLM-driven optimization is composing edits rather than knowing directions, then a zeroth-order method that recombines edits should (1) converge faster, (2) cost fewer tokens, and (3) reach better final solutions than first-order baselines---and, critically, its advantage should \emph{grow} as the solution-to-performance mapping becomes harder to reason about, since that is where first-order credit assignment degrades. We therefore study four task families---Circle Packing, Strategy Games, Agent System Design, and Path Planning---chosen to span combinatorial construction, adversarial and partially observable decision making, agentic reasoning, and planning, so that they stress both the reasoning and the optimization capabilities of the LLM. Unless otherwise stated, we use Qwen-Plus~\citep{qwen-plus} as the base LLM.

\subsection{Circle Packing}

\textbf{Task Description.} \quad 
The circle packing problem aims to place 26 circles within a unit square while maximizing the sum of their radii, subject to non-overlapping constraints and the requirement that all circles remain inside the square.
The task blends discrete placement decisions with continuous radius adjustments, making it a challenging benchmark: it features numerous local optima, and naive methods often become trapped in suboptimal configurations that use space inefficiently. This task provides a suitable testbed for our claim because good moves (e.g., shifting a row, regrowing a radius) are locally describable, yet their joint effect on the global sum is highly coupled---precisely the regime where bundling edits should hurt.


\textbf{Convergence Speed.} \quad
We measure the number of steps each method needs to reach a target performance of $2.60$; results are summarized in Fig.~\ref{fig:circlepacking}A. \ourmethod{} reaches the target with substantially fewer steps than all baselines. The reason is structural rather than incidental: by recombining edits, each step retains locally beneficial change instead of discarding a whole batch whenever one edit misfires, so progress compounds across steps rather than stalling. This translates into strong sample efficiency, which matters most when interaction budgets are tight---a common constraint when deploying LLM-driven optimization under time or compute limits.

\textbf{Token Consumption.} \quad
Faster convergence would be of little practical value if each step were expensive, so we next measure the average token usage per step over a 100-step run (Fig.~\ref{fig:circlepacking}B); tokens directly reflect runtime cost in LLM-based applications. \ourmethod{} consistently requires the fewest tokens per iteration. The saving is a direct consequence of being zeroth-order: because \ourmethod{} neither feeds in scored examples nor asks the model to reason over performance, its prompts omit the long example contexts and chain-of-thought analyses that first-order methods rely on, and each update is produced in a single edit-generation pass. Efficiency here is thus not an engineering trick but a built-in benefit of removing first-order reasoning.

\textbf{Learning Curve.} \quad
To see \emph{how} the methods reach their solutions, we compare learning curves over 100 steps (Fig.~\ref{fig:circlepacking}C). \ourmethod{} shows a pronounced early acceleration, improving far sooner than competitors, which progress slowly and incrementally. This shape is exactly what our diagnosis predicts: the model already retrieves strong task-relevant moves from the outset, so once selection prevents good moves from being vetoed, large gains are available immediately rather than having to be slowly inferred from accumulated score history.

\subsection{Strategy Games}
\label{sec:exp-strategy}

The previous task is fully deterministic. We now raise the difficulty of credit assignment along two axes---long-horizon interaction and partial observability---using two games, which lets us test the prediction that \ourmethod's advantage widens as reasoning from scores gets harder.

\textbf{Othello.} \quad
Othello is a classic two-player deterministic board game on an $8 \times 8$ grid. At each turn a player places a disc to flip the opponent's pieces, and the outcome depends heavily on long-horizon positional planning. Solving the task requires evaluating board configurations, anticipating multi-step consequences, and countering an adaptive opponent, making it a strong benchmark for strategic reasoning and adversarial foresight. We use win-rate as the performance signal and design two opponents (details in Appendix); we denote the task against the hard opponent as Othello-H and against the easy one as Othello-E.

\textbf{Battleship.} \quad
Battleship is a probabilistic strategy game in which a player must locate hidden ships on a grid through a sequence of queries. Unlike Othello, the environment is uncertain and only partially observable: the agent must maintain and continuously update a belief over feasible ship configurations (details in Appendix). Strong performance hinges on probing strategies that maximize information gain, balancing exploration and exploitation, and revising plans as new evidence arrives. Performance is measured as the ratio of successful hits to total shots. We define two difficulty levels by grid size, Battleship-H (hard) and Battleship-E (easy).


\begin{table}[t]
\centering
\caption{Results on Strategy Games.}
\label{tab:strategy-games}
\resizebox{\columnwidth}{!}{
\begin{tabular}{lccccc}
\toprule
\textbf{Task} 
& \textbf{\ourmethod} 
& \textbf{ShinkaEvolve} 
& \textbf{OpenEvolve} 
& \textbf{EoH} 
& \textbf{Best of N} \\
\midrule
Othello-E      & \textbf{0.97} & 0.90 & 0.93 & 0.30 & 0.36 \\
Othello-H      & \textbf{0.85} & 0.10 & 0.15 & 0.00 & 0.06 \\
Battleship-E   & \textbf{0.34} & 0.18 & 0.20 & 0.02 & 0.08 \\
Battleship-H   & \textbf{0.28} & 0.14 & 0.14 & 0.00 & 0.04 \\
\midrule
\textbf{Average} & \textbf{0.61} & 0.33 & 0.36 & 0.08 & 0.14 \\
\bottomrule
\end{tabular}}
\end{table}

\textbf{Performance.} \quad We run the optimizer for 60 steps on Othello and 20 steps on Battleship, and report final scores in Table~\ref{tab:strategy-games}, comparing against EoH, Best-of-N, OpenEvolve, and ShinkaEvolve under identical settings. \ourmethod{} achieves the best performance on all four tasks. More telling than the wins is their pattern. On the easy variants the field is close (Othello-E: $0.97$ vs.\ $0.90$--$0.93$ for the evolutionary baselines), but on the hard variants the gap explodes: on Othello-H \ourmethod{} reaches $0.85$ while \emph{every} competitor collapses to at most $0.15$. This is the signature our thesis predicts. On Othello-E the score-to-performance mapping is benign enough that first-order reasoning still works, so all reasonable methods cluster together; on Othello-H the mapping becomes deeply entangled, first-order direction-finding degenerates, and only a method that bypasses it---recombining the model's prior knowledge---keeps improving. The same trend holds on the partially observable Battleship ($0.34$/$0.28$ vs.\ at most $0.20$/$0.14$). Averaged over the four tasks, \ourmethod{} scores $0.61$, well above OpenEvolve ($0.36$) and ShinkaEvolve ($0.33$); notably, the largest contribution to this margin comes from the hard settings, which is precisely where the value of avoiding first-order reasoning should be greatest.


\begin{figure}[t]
    \centering
    \includegraphics[
        width=0.5\textwidth,
        keepaspectratio
    ]{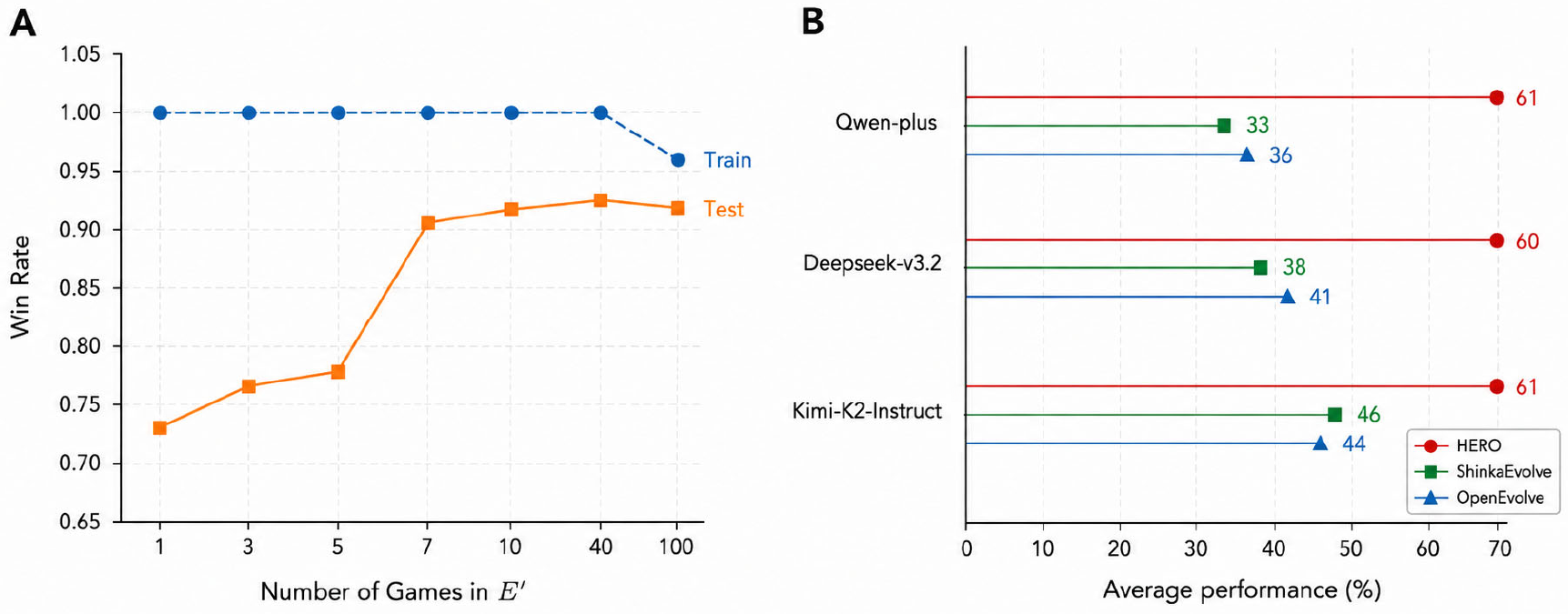}
    \caption{Analysis on Strategy Games. \textbf{A.} Effect of the surrogate evaluator $E'$ with different numbers of replicate games. \textbf{B.} Average performance of \ourmethod{}, ShinkaEvolve, and OpenEvolve across three LLMs.}
    \label{fig:rst-strategy}
\end{figure}

\textbf{Choice of Evaluator and Overfitting.} \quad A possible objection is that \ourmethod{} simply trades expensive reasoning for expensive evaluation, since selection (Eq.~\ref{eq:selectopm}) calls $f$ repeatedly. In Othello, the faithful evaluator runs 1{,}000 games per candidate, which is costly, so we replace it with a cheap surrogate $E'$ and ask how coarse $E'$ can be before selection suffers. Recall the method's claim that selection needs only the correct \emph{ranking} of subsets, not accurate absolute scores; Fig.~\ref{fig:rst-strategy}A tests this directly by sweeping the grain of $E'$ (number of games used to estimate the win-rate, horizontal axis) against win-rate (vertical axis). The training win-rate (blue) stays close to $1.0$ even with very few games, and the test win-rate (orange), though lower with a handful of games, rises rapidly and plateaus, with negligible gains beyond a small budget. The interpretation is that ranking is far more robust to estimation noise than the scores themselves: even a noisy $E'$ usually orders a clearly-good subset above a clearly-bad one, which is all selection requires. \ourmethod{} therefore attains accurate final performance without expensive large-scale evaluation, confirming that its reliance on $f$ is cheap in practice rather than a hidden cost.

\textbf{Choice of LLM.} \quad If the gains came from a quirk of one model, our claim about LLMs' latent knowledge would be weak; we therefore repeat the comparison across three base models in Fig.~\ref{fig:rst-strategy}B: Qwen-Plus, Deepseek-v3.2, and Kimi-K2-Instruct (full per-task breakdown in Appendix B). \ourmethod{} leads consistently---$61\%$ vs.\ $33\%$/$36\%$ on Qwen-Plus, $60\%$ vs.\ $38\%$/$41\%$ on Deepseek-v3.2, and $61\%$ vs.\ $46\%$/$44\%$ on Kimi-K2-Instruct---and its lead persists even on the weaker base model, indicating that the relevant directional knowledge is broadly present across models and that \ourmethod's role is to unlock it rather than to supply it. \ourmethod{} is also more stable across runs, whereas the baselines show larger variability and lower ceilings, consistent with the brittleness that all-or-nothing acceptance introduces.

\subsection{LLM Agent System Design}

The previous tasks optimize a single artifact. We now test whether the same mechanism improves an \emph{agent scaffold}, and whether the improvements transfer across model scales. We follow the protocol of \citet{lange2025shinkaevolve} on the AIME benchmark: the optimizer is run on AIME 2025~\citep{AIME2025} and evaluated on held-out AIME 2024~\citep{AIME2024} and AIME 2023~\citep{AIME2023}. Optimization runs for 10 steps using Qwen3-0.6B to generate candidate scaffolds; to probe generalization, we then evaluate the optimized scaffold on the larger Qwen3-7B model. Alongside accuracy, we report the average number of LLM queries per method as a proxy for computational cost.

Figure~\ref{fig:aime-tradeoff} shows that all optimization-based methods beat the base agent, lifting average accuracy from $18.5\%$ to over $27\%$; \ourmethod{} is best at $28.7\%$ ($33.3\%$ on 2024, $24.1\%$ on 2023). Two aspects support our thesis. First, the optimized scaffold, discovered on AIME 2025 with a tiny 0.6B model, still wins when transferred to a 7B model and to earlier years---improvements that are structural and model-agnostic, as expected if \ourmethod{} is selecting genuinely good design choices rather than overfitting to one evaluator. Second, the efficiency gap is large: \ourmethod{} needs only $2.15$ LLM calls on average, versus $4$ for ShinkaEvolve and $5$ for OpenEvolve, because it neither replays scored exemplars nor runs reasoning chains over them. The result is a favorable accuracy--efficiency trade-off that scales to practical settings.

         

\begin{figure}[t]
    \centering
    \includegraphics[width=0.8\linewidth]{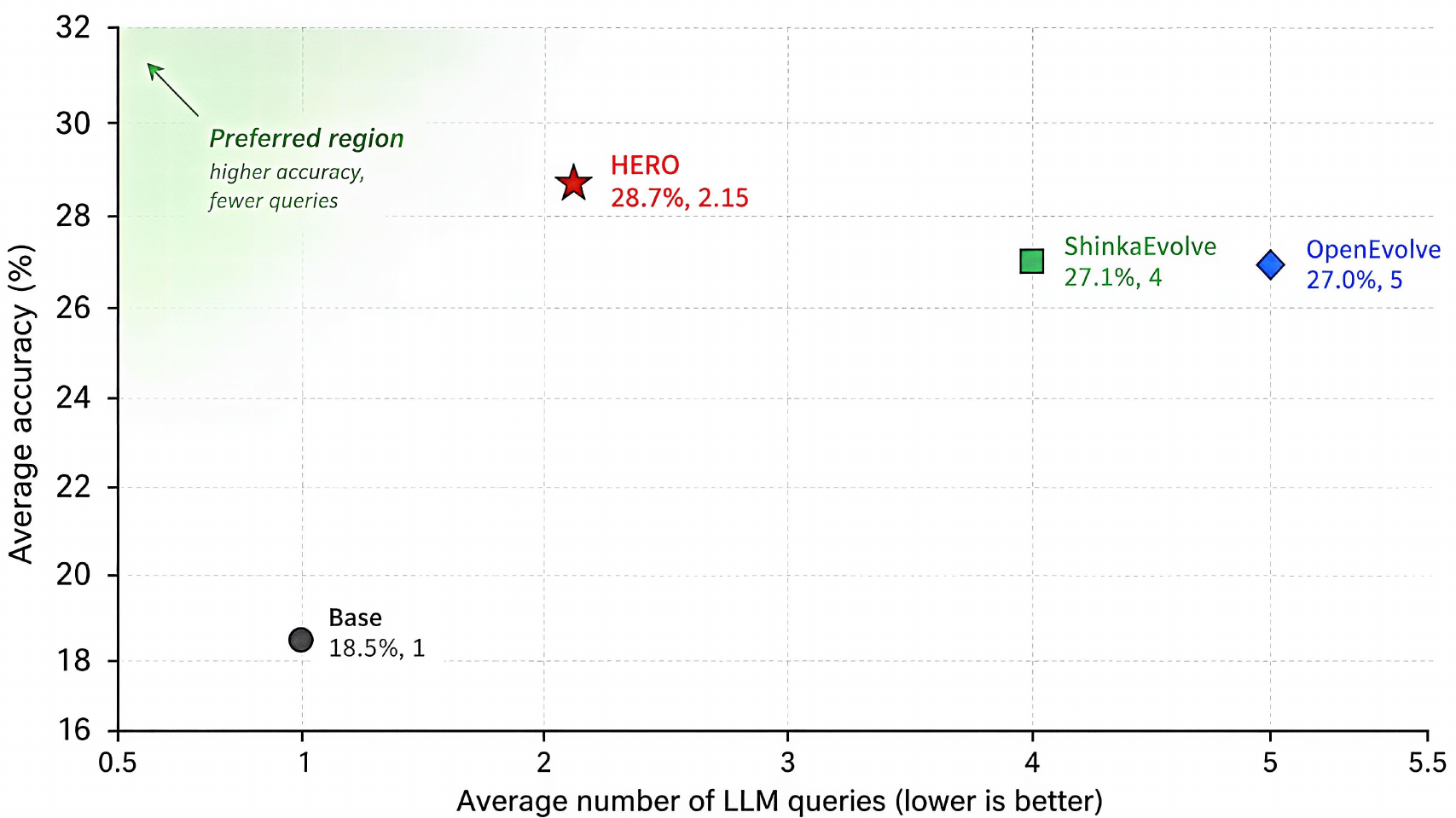}
    \caption{Accuracy--efficiency trade-off on AIME. Each point reports average accuracy over AIME 2024 and 2023 together with the average number of LLM queries. Points closer to the upper-left corner are preferred.}
    \label{fig:aime-tradeoff}
\end{figure}

\subsection{Path Planning}

Our final task moves from open-ended construction to a problem with a known optimum, which lets us separate \emph{solution quality} from \emph{search efficiency}. Grid-based robot path planning~\citep{wang2019path,wang2021path,jahanshahi2018robot} guides an agent from a start cell to a target cell on a two-dimensional lattice, where each cell is either traversable or occupied. The objective is to find a collision-free path of adjacent traversable cells that avoids all obstacles. Depending on the setting, movement may be restricted to four-connected neighborhoods (up, down, left, right) or extended to eight-connected motion including diagonals. Each step incurs a cost---uniform in simple settings, but variable when diagonal motion or terrain-dependent constraints are present---so the challenge is to reach the goal while minimizing cumulative path cost. Because the environment is fully discretized, the problem is naturally formulated as graph search, with grid cells as nodes and valid transitions as edges; classical algorithms such as A* find feasible, cost-optimal paths via heuristics and cost structures.

Figure~\ref{fig:path-planning} reports \textbf{path length} (solution quality) and \textbf{search complexity} (planning cost, measured by visited states), comparing algorithms designed by \ourmethod{} and OpenEvolve against A* and PPO. On path length all methods tie---length 33/34 in Canyon and 40 in Double Door---so optimality is not in question, and the comparison reduces to efficiency. There the methods diverge sharply: in Canyon, A* and PPO visit 264 and 256 states while \ourmethod{} visits only 142, and in the harder Double Door the gap widens further---A* and PPO need 320 and 314 while \ourmethod{} needs just 132, about a third of the classical cost. Tellingly, OpenEvolve, the other LLM optimizer, is the \emph{least} efficient (361 and 491): merely using an LLM does not buy efficiency. The advantage comes from \ourmethod{} composing the model's heuristic knowledge into a search procedure that prunes unnecessary exploration, and once again it grows with task difficulty.


\begin{figure}[t]
    \centering
    \includegraphics[width=0.86\linewidth]{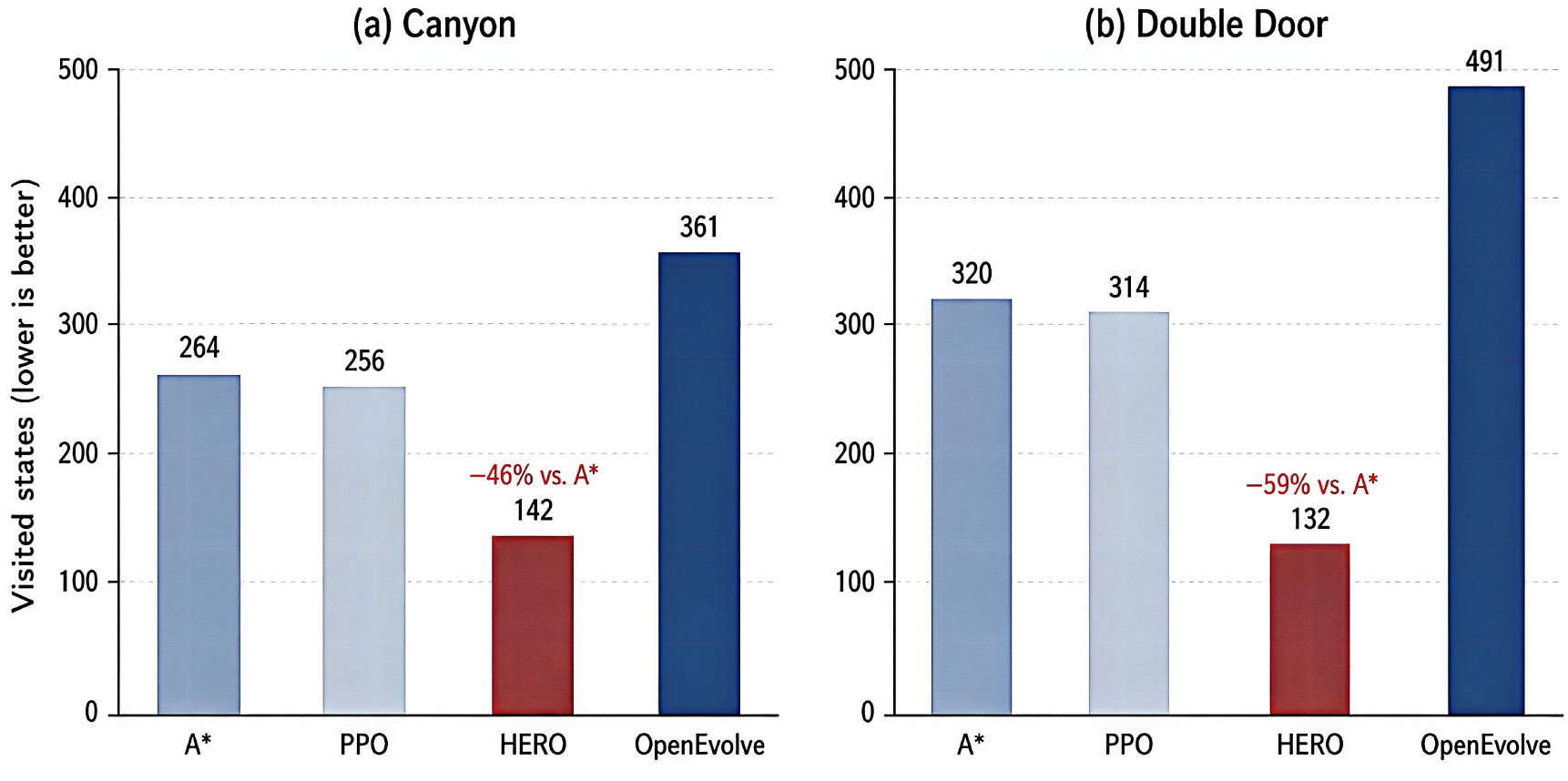}
    \caption{Path planning performance. All methods achieve comparable path quality, while \ourmethod{} substantially reduces search complexity measured by visited states. The path lengths are reported below each environment, and the bars show the number of visited states.}
    \label{fig:path-planning}
\end{figure}

\section{Analysis of \ourmethod}

Having shown \emph{that} \ourmethod{} works, we now examine \emph{why}, and whether its simplicity is a limitation or a feature. We analyze two representative game benchmarks, Othello and Battleship, which differ in decision structure and feedback: Othello emphasizes local strategy composition and long-horizon value evaluation in a sequential game, whereas Battleship emphasizes search and exploration under imperfect information. Studying both lets us probe \ourmethod{} across complementary, strategy-driven and search-driven regimes, and to ask whether common ``add-ons'' would improve it.



\begin{figure}[t]
    \centering
    \includegraphics[width=0.76\linewidth]{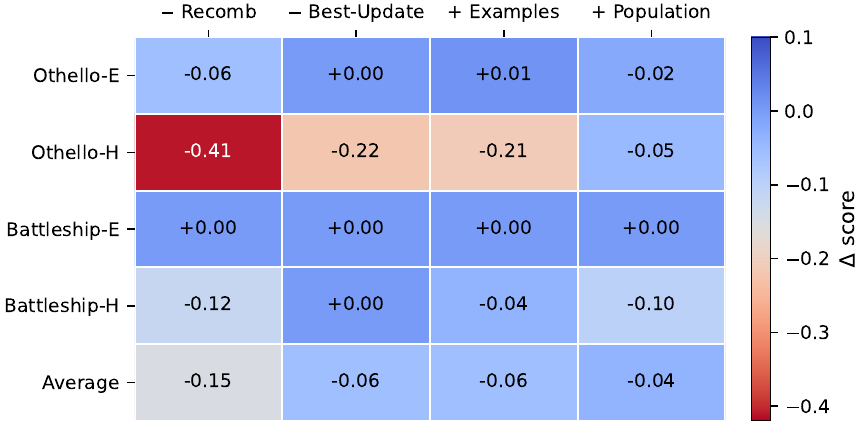}
    \caption{Ablation study on Strategy Games. Each cell reports the performance change relative to \ourmethod{}, computed as the variant score minus the full \ourmethod{} score. Negative values indicate performance drops compared with \ourmethod{}, while positive values indicate improvements.}
    \label{fig:ablation}
\end{figure}

\subsection{Ablation Study}
We first remove \ourmethod's two core components---the \emph{recombination mechanism} and the \emph{best-update mechanism}---to test whether each is necessary, and we report both alongside two common enhancements (in-context examples and a larger population) for context. All results are in Figure~\ref{fig:ablation}.

Removing recombination---reverting to applying the raw bundled output---is the most direct test of the weakest-link hypothesis. As shown in Fig.~\ref{fig:ablation}, this ablation causes the largest average drop, reducing the score by $0.15$ relative to the full \ourmethod{} configuration. The effect is especially pronounced on Othello-H, where performance drops from $0.85$ to $0.44$. Battleship-H also decreases from $0.28$ to $0.16$, while Battleship-E remains saturated at $0.34$. This pattern supports the claim that recombination is most valuable when the generated edit bundle is likely to contain both beneficial and harmful edits.

Removing best-update also weakens the method, but less severely. The average score decreases from $0.61$ to $0.56$, with the main loss appearing on Othello-H, where performance drops from $0.85$ to $0.63$. Battleship-H remains unchanged in the best-of-seeds score, but this does not negate the role of best-update; rather, it suggests that its stabilizing effect is most visible in strategy-driven settings such as Othello. Overall, recombination provides the main protection against weakest-link failures, while best-update stabilizes the trajectory around the best-so-far program.

\subsection{Optimality Verification}
A natural worry is that such a minimal design leaves easy gains on the table. We test this by adding two enhancements widely believed to help LLM optimizers---extra in-context examples and a larger program population---and asking whether either improves on the full method (Figure~\ref{fig:ablation}).


The answer is mixed at the individual-task level but clear at the average level. Adding examples slightly improves Othello-E from $0.97$ to $0.98$, but it hurts the harder settings, reducing Othello-H from $0.85$ to $0.64$ and Battleship-H from $0.28$ to $0.24$. Its average score is therefore $0.55$, below the full \ourmethod{} score of $0.61$. This suggests that additional examples can help in benign settings, but they do not provide a robust improvement across tasks.

Increasing the population shows a similar pattern. It remains competitive on Othello-H, reaching $0.80$, but it reduces Battleship-H from $0.28$ to $0.18$ and yields an average score of $0.57$, still below the full method. Thus, neither enhancement improves the overall accuracy--efficiency trade-off. The results support the design choice of keeping \ourmethod{} minimal: its gains come primarily from selectively recombining heterogeneous edits and anchoring the search to the best-so-far program, rather than from adding more context or maintaining a larger population.

\section{Conclusion}

We introduced \ourmethod{}, which improves programs by generating atomic edits without performance feedback and selecting their best recombination using scalar scores. This avoids unreliable score reasoning and the weakest-link effect.



\bibliography{aaai2027}


\input{appendix}

\end{document}

%% file: appendix.tex
\newpage
\appendix

\section{Related Works: Discussion on First-Order Information}

The contribution of this paper is best understood against the assumption it questions---that reasoning from ground-truth performance is necessary for LLM-based optimization. We therefore organize related work around \emph{first-order information}: how it is used, when it genuinely helps, and where it quietly breaks down.

\subsection{First-Order Information in LLM Optimizers}
Following the convention of numerical optimization, we treat the ground-truth performance of a solution as zeroth-order information and any \emph{reasoning} built on that performance as first-order information~\citep{yuksekgonul2025optimizing,zhanghessiangrad}. The idea that LLMs can act as optimizers dates back to early work showing they can iteratively improve solutions when shown past attempts and their scores~\citep{yang2023large}, and first-order information has since been viewed as the key signal for choosing the next direction.

Two lines of work operationalize this idea, and it is useful to see that they sit on a spectrum of how \emph{explicitly} they use first-order information. The first makes it explicit. TextGrad~\citep{yuksekgonul2025optimizing}, inspired by the autodiff mechanism of PyTorch~\citep{paszke2019pytorch}, asks the model to analyze its previous outputs against performance feedback and emit improvement suggestions that act as a ``textual gradient.'' The second uses it implicitly. Evolutionary LLM frameworks~\citep{novikov2025alphaevolve,FunSearchNature2024,liu2024evolution} maintain a population of scored candidates and sample parents from it~\citep{brahmachary2024leo,pourcel2025self}, so the optimizer need not fixate on the most recent solution; first-order reasoning still enters, however, because the model is shown several past solutions \emph{with their scores} and asked to reason across them when proposing new ones~\citep{guoconnecting,ye2024reevo}. Despite their differences, both lines share a premise: the model should look at performance and reason about it.

\subsection{Potential and Limitation of First-Order Information}

First-order information is genuinely useful when the link between a solution and its score is easy to reason about\cite{wang2023closing,fu2019recognition,shen2023structvpr,fu2023momentum,fu2024understanding,li2025regnav}. Its core function is to let the model associate a candidate with its performance and adjust accordingly; when that mapping is simple and near-decomposable, the association is reliable. In a poem-generation task~\citep{cheng2023llf}, for example, the model can directly see that a line is too long and shorten it. And in settings where the model has little intrinsic prior about the task~\citep{yang2023large}, first-order signals may be the \emph{only} source of direction, making them indispensable.

The picture changes when evaluation is a complex, multi-stage process---as in multi-step games~\citep{kuang2025learning,bosio2025synthesizing}. There the final score aggregates many interacting decisions, so the causal path from any single component to the outcome is long and entangled. Reasoning back from the score to ``what to fix'' becomes a credit-assignment problem with no unique answer, and the resulting direction is unreliable. This is the same failure our analysis in Appendix~\ref{app:fail-first-order} formalizes, and it is, paradoxically, the regime where strong optimization is most needed.

\subsection{Our Perspective}

We draw a different conclusion from these limitations. Rather than working harder to extract directions from first-order information, we observe that for many complex tasks the LLM already holds substantial internal knowledge about what makes a solution good, by virtue of understanding the task itself. When that is true, performance need not be \emph{reasoned about} at all; it need only serve as a scalar criterion for comparing candidates---zeroth-order information in our terminology. The remaining obstacle is then not knowing the direction but composing the model's many partial directions into a coherent improvement, which is exactly the weakest-link effect that \ourmethod{} addresses.

We stress what this perspective does and does not claim. We do not argue that first-order information is unimportant or obsolete; in low-prior tasks it remains essential. Rather, by showing how far zeroth-order information alone can go, we aim to surface an underappreciated strength---the knowledge already inside the model---and to prompt a more careful, cost-aware account of \emph{when} first-order reasoning is worth its overhead and \emph{how} it should be combined with the model's prior.

\section{Design Details}

We instantiate the two phases with concrete choices that preserve the abstractions above while keeping the method strictly zeroth-order.

\paragraph{Generation.}
We realize the generation query as $\mathrm{LLM}(\mathrm{Task},p)$: the model receives the task description and the current program $p$ only, with the performance of $p$ withheld. Following prior work, the task description carries the essential problem specification plus optional hints and is held \emph{fixed} throughout optimization, so that any change in behavior is attributable to the search rather than to a shifting prompt. Rather than regenerating the whole program, we adopt the \texttt{SEARCH}/\texttt{REPLACE} edit format of \citet{novikov2025alphaevolve}: the model emits several \texttt{SEARCH}/\texttt{REPLACE} pairs, each a small, localized modification. We constrain the \texttt{SEARCH} regions to be pairwise non-overlapping. This is a simple sufficient condition for the commutativity requirement~\eqref{eq:commute}---edits touching disjoint code regions cannot interfere---and it makes the resulting atomic edits genuinely heterogeneous, as the Generation phase demands.

\paragraph{Selection.}
Let $E$ denote the evaluator that returns the scalar performance used in place of $f$ (e.g., a test harness or simulator). Selection proceeds in two sequential steps.
\textit{(1) Error filtering.} We apply each edit $\alpha_i$ individually and discard any that triggers an execution error, yielding the valid set $A_{\mathrm{ne}}\subseteq A$. This removes obviously broken directions and, crucially, prevents a single invalid edit from contaminating every subset that would contain it.
\textit{(2) Performance optimization.} Over $A_{\mathrm{ne}}$ we solve Eq.~\eqref{eq:selectopm}. Exact selection is a subset search whose cost grows exponentially in $|A_{\mathrm{ne}}|$, so we match the budget to the regime: when $|A_{\mathrm{ne}}|$ is small or $E$ is cheap, we enumerate all subsets; when $|A_{\mathrm{ne}}|$ is large, we evaluate randomly sampled subsets; and when $E$ is expensive, we construct a cheap surrogate $E'\approx E$ and optimize against $E'$. The surrogate is sound for a specific reason: by Eq.~\eqref{eq:selectopm}, selection depends only on the relative \emph{ranking} of subsets, not on accurate absolute scores, so any coarse but approximately order-preserving $E'$ suffices. We verify this empirically in Section~\ref{sec:exp-strategy} via the evaluator-grain analysis.

\paragraph{Selection.}
Let $E$ denote the scalar evaluator used in place of $f$ (e.g., a test harness or
simulator). Selection uses this signal only as a comparison oracle and proceeds in two
steps.
\textit{(1) Error filtering.}
We first apply each edit $\alpha_i$ individually and discard any edit that causes an
execution error, obtaining the non-error set $A_{\mathrm{ne}}\subseteq A$. This prevents a
single invalid edit from corrupting every recombined subset that contains it.
\textit{(2) Budgeted subset ranking.}
Over $A_{\mathrm{ne}}$, exact selection would enumerate all subsets as in
Eq.~\eqref{eq:selectopm}, whose cost is exponential in $|A_{\mathrm{ne}}|$. We therefore
solve a budgeted approximation:
\begin{equation}
\label{eq:budgeted-selection}
\begin{gathered}
S^\star
\in
\arg\max_{S\in\mathcal{C}_B(A_{\mathrm{ne}})}
\widehat{E}_{m_S}\!\left(A_S(x_{t-1})\right), \\
\sum_{S\in\mathcal{C}_B(A_{\mathrm{ne}})} m_S \le B .
\end{gathered}
\end{equation}
where $\mathcal{C}_B(A_{\mathrm{ne}})\subseteq 2^{A_{\mathrm{ne}}}$ is the subset pool
considered under budget $B$, $m_S$ is the number of evaluator calls assigned to subset
$S$, and
$\widehat{E}_{m}(x)=\frac{1}{m}\sum_{r=1}^{m}\widetilde{E}_r(x)$ denotes the empirical
score estimated either with the faithful evaluator $E$ or, when evaluation is expensive,
with a cheaper surrogate $E'\approx E$. This formulation captures the practical
breadth--reliability trade-off: the budget can be spent on evaluating more subsets
(larger $\mathcal{C}_B$) or on reranking promising subsets more reliably (larger
$m_S$). When the budget is sufficient, $\mathcal{C}_B(A_{\mathrm{ne}})=2^{A_{\mathrm{ne}}}$
and Eq.~\eqref{eq:budgeted-selection} recovers exact selection; otherwise, we sample a
subset pool and optionally allocate additional evaluations to top-ranked candidates.
Because selection depends only on the relative ranking of subsets, an approximately
order-preserving surrogate is sufficient. In all cases, the performance signal is used
only for subset comparison: it is never provided to the LLM during edit generation and
is never used as reasoning context.

\section{A Analysis of First-Order Information}
\label{app:fail-first-order}

This appendix formalizes the claim made in the main text: once the mapping from a candidate solution to its performance is unrestricted, finitely many evaluations cannot determine the performance of an unseen candidate, so any direction inferred by reasoning over scores is unidentifiable. We proceed from definitions and assumptions to a core lemma, the main impossibility theorem, and two corollaries, and we close with remarks on intuition, tightness, and the implication for first-order optimization.

\subsection{Setup and Definitions}

Throughout, $\mathcal{X}$ denotes the space of candidate solutions and $\mathcal{Y}\subseteq\mathbb{R}^{+}$ the space of (performance) labels, with $|\mathcal{Y}|\ge 2$. We use $[k]=\{1,\dots,k\}$.

\begin{definition}[Score--decoder model]
\label{def:score-decoder}
Fix a measurable \emph{latent score} $s:\mathcal{X}\to\mathbb{R}$. Let $\mathcal{M}$ be the set of all measurable \emph{decoders} $\phi:\mathbb{R}\to\mathcal{Y}$. The induced hypothesis class is
\[
  \mathcal{F}_s \;=\; \{\, \phi\circ s \;:\; \phi\in\mathcal{M} \,\}.
\]
The target function $f^{\star}=\phi^{\star}\circ s\in\mathcal{F}_s$ is fixed but unknown; here the latent score $s$ is common to every hypothesis, and only the decoder $\phi^{\star}$ is unknown.
\end{definition}

\begin{definition}[Sample, observed scores, version space]
\label{def:consistency}
A size-$k$ sample is $S_k=\{(x_i,y_i)\}_{i=1}^{k}$ with $y_i=f^{\star}(x_i)$, and its \emph{observed score set} is $Z_{\mathrm{obs}}=\{s(x_i):i\in[k]\}\subset\mathbb{R}$. A decoder $\phi\in\mathcal{M}$ is \emph{consistent} with $S_k$ if $\phi(s(x_i))=y_i$ for all $i\in[k]$. The \emph{version space} is the set of all consistent decoders, $V(S_k)=\{\phi\in\mathcal{M}:\phi(s(x_i))=y_i\ \text{for all }i\in[k]\}$.
\end{definition}

We make the following assumptions explicit. Assumptions (A1)--(A3) underlie all results below; (A4) is invoked only for the probabilistic statement of Corollary~\ref{cor:as}.
\begin{itemize}
  \item[(A1)] \textbf{Unrestricted decoder.} $\mathcal{M}$ is the class of \emph{all} measurable maps $\mathbb{R}\to\mathcal{Y}$: no smoothness, monotonicity, or parametric restriction is imposed on $\phi$.
  \item[(A2)] \textbf{Latent observability.} The score $s$ is fixed and shared across $\mathcal{F}_s$, but its values are never observed directly; the learner observes only labels $y=\phi^{\star}(s(x))$.
  \item[(A3)] \textbf{Finite data.} The sample size $k$ is finite, so $Z_{\mathrm{obs}}$ is a finite subset of $\mathbb{R}$.
  \item[(A4)] \textbf{Non-atomic query scores.} When the query point $X_q$ is random, the law of $s(X_q)$ is non-atomic (e.g.\ it admits a density), so that $\Pr[s(X_q)=z]=0$ for every fixed $z\in\mathbb{R}$.
\end{itemize}

The next lemma isolates the construction on which all subsequent results rest: outside the finitely many observed scores, a measurable decoder may be assigned freely without violating consistency.

\begin{lemma}[Free assignment off the observed scores]
\label{lem:free-assign}
Assume \textnormal{(A1)--(A3)}. Let $z_1,\dots,z_m\in\mathbb{R}\setminus Z_{\mathrm{obs}}$ be distinct, and let $b_1,\dots,b_m\in\mathcal{Y}$ be arbitrary. Then there exists $\phi_{\mathrm{adv}}\in V(S_k)$ with $\phi_{\mathrm{adv}}(z_j)=b_j$ for every $j\in[m]$.
\end{lemma}

\begin{proof}
We construct $\phi_{\mathrm{adv}}$ directly and verify each required property in turn.

\emph{(i) Consistency on observed scores.} For $z\in Z_{\mathrm{obs}}$, define $\phi_{\mathrm{adv}}(z)=y_i$ for any index $i\in[k]$ with $s(x_i)=z$. This is well defined: if $s(x_i)=s(x_{i'})$, then by Definition~\ref{def:score-decoder} $y_i=\phi^{\star}(s(x_i))=\phi^{\star}(s(x_{i'}))=y_{i'}$, so the assigned value is independent of the chosen index. By construction $\phi_{\mathrm{adv}}(s(x_i))=y_i$ for all $i\in[k]$, which is exactly the consistency condition of Definition~\ref{def:consistency}.

\emph{(ii) Prescribing the target values.} By (A3), $Z_{\mathrm{obs}}$ is finite and hence closed, and the points $z_1,\dots,z_m\in\mathbb{R}\setminus Z_{\mathrm{obs}}$ are distinct. We may therefore choose pairwise-disjoint open intervals $I_1,\dots,I_m$ with $z_j\in I_j$ and $I_j\cap Z_{\mathrm{obs}}=\emptyset$ for every $j\in[m]$. Set $\phi_{\mathrm{adv}}(z)=b_j$ for all $z\in I_j$.

\emph{(iii) Extension and measurability.} Fix any $y_0\in\mathcal{Y}$ and set $\phi_{\mathrm{adv}}(z)=y_0$ for $z\in\mathbb{R}\setminus\bigl(Z_{\mathrm{obs}}\cup\bigcup_{j=1}^{m} I_j\bigr)$. Then $\phi_{\mathrm{adv}}$ is constant on each of finitely many measurable pieces---a finite point set, finitely many disjoint open intervals, and their complement---hence is a simple function and therefore measurable; by (A1), $\phi_{\mathrm{adv}}\in\mathcal{M}$.

By (i) the decoder is consistent, so $\phi_{\mathrm{adv}}\in V(S_k)$; by (ii) and $z_j\in I_j$ we have $\phi_{\mathrm{adv}}(z_j)=b_j$ for every $j\in[m]$.
\end{proof}

\subsection{Impossibility Results}

\begin{theorem}[Pointwise unidentifiability]
\label{thm:main}
Assume \textnormal{(A1)--(A3)}. Let $x_q\in\mathcal{X}$ be a query point with unobserved latent score, i.e.\ $s(x_q)\notin Z_{\mathrm{obs}}$. Then for every target label $y_{\mathrm{target}}\in\mathcal{Y}$ there exists a hypothesis $f_{\mathrm{adv}}=\phi_{\mathrm{adv}}\circ s\in\mathcal{F}_s$ that is consistent with $S_k$ and satisfies $f_{\mathrm{adv}}(x_q)=y_{\mathrm{target}}$. Consequently $f^{\star}(x_q)$ is not determined by $S_k$: no algorithm that observes only $S_k$ can exclude any label in $\mathcal{Y}$ as the value at $x_q$.
\end{theorem}

\begin{proof}
We argue directly using Lemma~\ref{lem:free-assign}, then conclude by a two-hypotheses (indistinguishability) argument. Applying Lemma~\ref{lem:free-assign} with the single point $z_1=s(x_q)\in\mathbb{R}\setminus Z_{\mathrm{obs}}$ and value $b_1=y_{\mathrm{target}}$ yields a decoder $\phi_{\mathrm{adv}}\in V(S_k)$ with $\phi_{\mathrm{adv}}(s(x_q))=y_{\mathrm{target}}$. The hypothesis $f_{\mathrm{adv}}=\phi_{\mathrm{adv}}\circ s$ then lies in $\mathcal{F}_s$ by Definition~\ref{def:score-decoder}, is consistent with $S_k$ because $\phi_{\mathrm{adv}}\in V(S_k)$ (Definition~\ref{def:consistency}), and satisfies $f_{\mathrm{adv}}(x_q)=\phi_{\mathrm{adv}}(s(x_q))=y_{\mathrm{target}}$.

For the final claim, suppose some algorithm, on input $S_k$, asserted $f^{\star}(x_q)\neq y$ for a particular $y\in\mathcal{Y}$. Instantiating the construction with $y_{\mathrm{target}}=y$ produces $f_{\mathrm{adv}}\in\mathcal{F}_s$ consistent with $S_k$ and with $f_{\mathrm{adv}}(x_q)=y$. The two hypotheses $f^{\star}$ and $f_{\mathrm{adv}}$ agree on all of $S_k$ and are thus indistinguishable to any procedure whose input is $S_k$; since $f_{\mathrm{adv}}(x_q)=y$ is admissible, the assertion $f^{\star}(x_q)\neq y$ is unjustified. As $y\in\mathcal{Y}$ was arbitrary, no label can be excluded.
\end{proof}

The barrier extends from single labels to \emph{comparisons}, which is the form first-order reasoning actually relies on.

\begin{corollary}[Preference unidentifiability]
\label{cor:pref}
Assume \textnormal{(A1)--(A3)}. Let $x',x''\in\mathcal{X}$ have distinct, unobserved latent scores, i.e.\ $s(x')\neq s(x'')$ and $s(x'),s(x'')\notin Z_{\mathrm{obs}}$. Then both strict orderings are realizable by consistent hypotheses: there exist $f_1,f_2\in\mathcal{F}_s$, each consistent with $S_k$, such that $f_1(x')<f_1(x'')$ and $f_2(x')>f_2(x'')$. Hence no algorithm observing only $S_k$ can determine which of $x',x''$ has the larger label.
\end{corollary}

\begin{proof}
Since $|\mathcal{Y}|\ge2$, fix $a,b\in\mathcal{Y}$ with $a<b$. The points $z_1=s(x')$ and $z_2=s(x'')$ are distinct and lie in $\mathbb{R}\setminus Z_{\mathrm{obs}}$, so Lemma~\ref{lem:free-assign} applies with $m=2$. Choosing $(b_1,b_2)=(a,b)$ gives a consistent decoder whose hypothesis $f_1$ satisfies $f_1(x')=a<b=f_1(x'')$; choosing $(b_1,b_2)=(b,a)$ gives a consistent $f_2$ with $f_2(x')=b>a=f_2(x'')$. Both hypotheses lie in $\mathcal{F}_s$ and are consistent with $S_k$ by Lemma~\ref{lem:free-assign}, so the observed data are compatible with either ordering.
\end{proof}

Theorem~\ref{thm:main} is stated for a fixed query point satisfying the genericity condition $s(x_q)\notin Z_{\mathrm{obs}}$. We now show that, for a randomly drawn query, this condition holds almost surely.

\begin{corollary}[Almost-sure unidentifiability]
\label{cor:as}
Assume \textnormal{(A1)--(A4)} and let the query $X_q$ be drawn independently of $S_k$ from a distribution whose pushforward under $s$ is non-atomic. Then $s(X_q)\notin Z_{\mathrm{obs}}$ with probability one, and hence the conclusion of Theorem~\ref{thm:main} holds almost surely.
\end{corollary}

\begin{proof}
By (A3) the set $Z_{\mathrm{obs}}=\{s(x_1),\dots,s(x_k)\}$ is finite. By (A4) and a union bound,
\[
\Pr\bigl[s(X_q)\in Z_{\mathrm{obs}}\bigr]\;\le\;\sum_{i=1}^{k}\Pr\bigl[s(X_q)=s(x_i)\bigr]\;=\;0 .
\]
On the complementary event $\{s(X_q)\notin Z_{\mathrm{obs}}\}$, which therefore has probability one, Theorem~\ref{thm:main} applies verbatim.
\end{proof}

\begin{remark}[Why the result holds, and why \textnormal{(A1)} is essential]
The obstruction is information-theoretic rather than computational: by Lemma~\ref{lem:free-assign} the labels constrain the decoder only on the finite set $Z_{\mathrm{obs}}$, leaving its values on $\mathbb{R}\setminus Z_{\mathrm{obs}}$ entirely unconstrained, so the version space $V(S_k)$ remains rich enough to realize any label at an unseen query. Assumption (A1) is precisely what enables this: with an unrestricted decoder, a constraint at an observed score cannot propagate to a nearby unobserved one. The result is tight in this sense---any restriction that couples nearby scores defeats the construction. If, for instance, $\phi$ were $L$-Lipschitz, then $|f_{\mathrm{adv}}(x_q)-y_i|\le L\,|s(x_q)-s(x_i)|$ for the nearest anchor $x_i$, so labels near observed scores could no longer be set arbitrarily; monotone or finite-dimensional parametric decoder families likewise restore identifiability. This delineates exactly when reasoning from performance is trustworthy: only when the score-to-performance map is suitably constrained.
\end{remark}

\begin{remark}[Implication for first-order optimization]
In our setting the label $y=f(x)$ is the \emph{performance} of candidate $x$, and the unknown decoder $\phi$ models a complex, unconstrained map from latent quality $s(x)$ to observed performance. Corollary~\ref{cor:pref} then states precisely that finite evaluation history cannot, in general, reveal which of two unseen candidates performs better. Because first-order reasoning operates by inferring such comparisons from observed $(x_i,y_i)$ pairs, it is unreliable in exactly this regime---which motivates \ourmethod's use of performance as a direct scalar selection criterion rather than as a signal to reason from.
\end{remark}

\begin{remark}[On the original probabilistic formulation; flagged correction]
\label{rem:flag}
An earlier statement of this result asserted that the conclusion holds ``with probability at least $1-\epsilon$ over the sampling of $S_k$'' for every $\epsilon>0$. We flag this as imprecise and have restated it above rather than retaining it, for two reasons. First, the relevant randomness is not the draw of $S_k$ but the position of the query score $s(X_q)$ relative to the finite set $Z_{\mathrm{obs}}$; conditioned on $S_k$, the construction is deterministic (Theorem~\ref{thm:main}) once $s(x_q)\notin Z_{\mathrm{obs}}$. Second, the construction in fact yields probability \emph{one}, not merely $1-\epsilon$, but only under a non-atomicity assumption on the score distribution that was absent from the original hypotheses; we have made it explicit as (A4) and isolated the probabilistic claim in Corollary~\ref{cor:as}. Without (A4) the claim can genuinely fail: if the law of $s$ has an atom and $s(x_q)=s(x_i)$ for some $i$, then consistency forces $f(x_q)=y_i$, the label \emph{is} identified, and no adversarial decoder exists. The genericity condition $s(x_q)\notin Z_{\mathrm{obs}}$ is therefore necessary, not cosmetic.
\end{remark}

\subsection{Experiments}

To further demonstrate that first-order information is little used in practice, we conduct a controlled experiment. We use OpenEvolve as the base optimizer and iteratively remove the in-context examples that carry first-order signal, observing how performance changes. We find that removing these examples leaves performance largely unaffected, indicating that the optimizer relies far less on first-order information than commonly assumed.



\section{Extra Results}
\label{app:extra-results}

This section reports the full per-task results underlying the cross-model study summarized in Section~\ref{sec:exp-strategy} (Fig.~\ref{fig:rst-strategy}B). To confirm that \ourmethod's advantage is not an artifact of a single backbone, we repeat the Strategy Games protocol of Section~\ref{sec:exp-strategy} on two additional base models---DeepSeek-v3.2 (Table~\ref{tab:strategy-games-deepseek}) and Moonshot-Kimi-K2-Instruct (Table~\ref{tab:strategy-games-kimi})---comparing against the two strongest evolutionary baselines, ShinkaEvolve and OpenEvolve, under identical settings (60 steps on Othello, 20 on Battleship). Together with the Qwen-Plus results in Table~\ref{tab:strategy-games}, these span three independently developed model families.

\begin{table}[t]
\centering
\caption{Per-task results on Strategy Games using DeepSeek-v3.2. The best result in each row is in \textbf{bold}.}
\label{tab:strategy-games-deepseek}
\begin{tabular}{lccc}
\toprule
\textbf{Task} & \textbf{\ourmethod}  & \textbf{ShinkaEvolve} & \textbf{OpenEvolve} \\
\midrule
Othello-E     & 0.94 & 0.94 & \textbf{0.96} \\
Othello-H     & \textbf{0.95} & 0.28 & 0.26 \\
Battleship-E  & \textbf{0.32} & 0.20 & 0.30 \\
Battleship-H  & \textbf{0.16} & 0.10 & 0.10 \\
\midrule
\textbf{Average} & \textbf{0.60} & 0.38 & 0.41 \\
\bottomrule
\end{tabular}
\end{table}
\begin{table}[t]
\centering
\caption{Per-task results on Strategy Games using Moonshot-Kimi-K2-Instruct. The best result in each row is in \textbf{bold}.}
\label{tab:strategy-games-kimi}
\begin{tabular}{lccc}
\toprule
\textbf{Task} & \textbf{\ourmethod}  & \textbf{ShinkaEvolve} & \textbf{OpenEvolve} \\
\midrule
Othello-E     & \textbf{0.95} & 0.94 & 0.94 \\
Othello-H     & \textbf{0.93} & 0.60 & 0.65 \\
Battleship-E  & \textbf{0.36} & 0.18 & 0.12 \\
Battleship-H  & \textbf{0.18} & 0.10 & 0.06 \\
\midrule
\textbf{Average} & \textbf{0.61} & 0.46 & 0.44 \\
\bottomrule
\end{tabular}
\end{table}

Two observations reinforce the conclusions of the main text. First, \ourmethod{} attains the best average win-rate on both backbones---$0.60$ versus $0.41$ (OpenEvolve) and $0.38$ (ShinkaEvolve) on DeepSeek-v3.2, and $0.61$ versus $0.44$ and $0.46$ on Kimi-K2-Instruct---mirroring the $0.61$ versus $0.36$/$0.33$ margin on Qwen-Plus (Table~\ref{tab:strategy-games}). The directional knowledge \ourmethod{} exploits is therefore broadly present across model families rather than peculiar to one. Second, the easy--hard pattern recurs within each backbone: on the easy variants (Othello-E, Battleship-E) the methods are close, and a baseline can occasionally edge ahead (e.g., OpenEvolve reaches $0.96$ on Othello-E with DeepSeek-v3.2), whereas on the hard variants the gap widens sharply---on Othello-H, \ourmethod{} scores $0.95$ and $0.93$ while no baseline exceeds $0.28$ and $0.65$, respectively. This is the same signature predicted by our analysis and observed on Qwen-Plus: the benefit of bypassing first-order reasoning is largest exactly where the score-to-performance mapping is hardest to invert, and it holds regardless of the underlying model.

\section{Detail of Task Design}

\subsection{Othello Game}
\label{subsec:app_othello_games}
\begin{figure}
    \centering
    \includegraphics[width=0.79\linewidth]{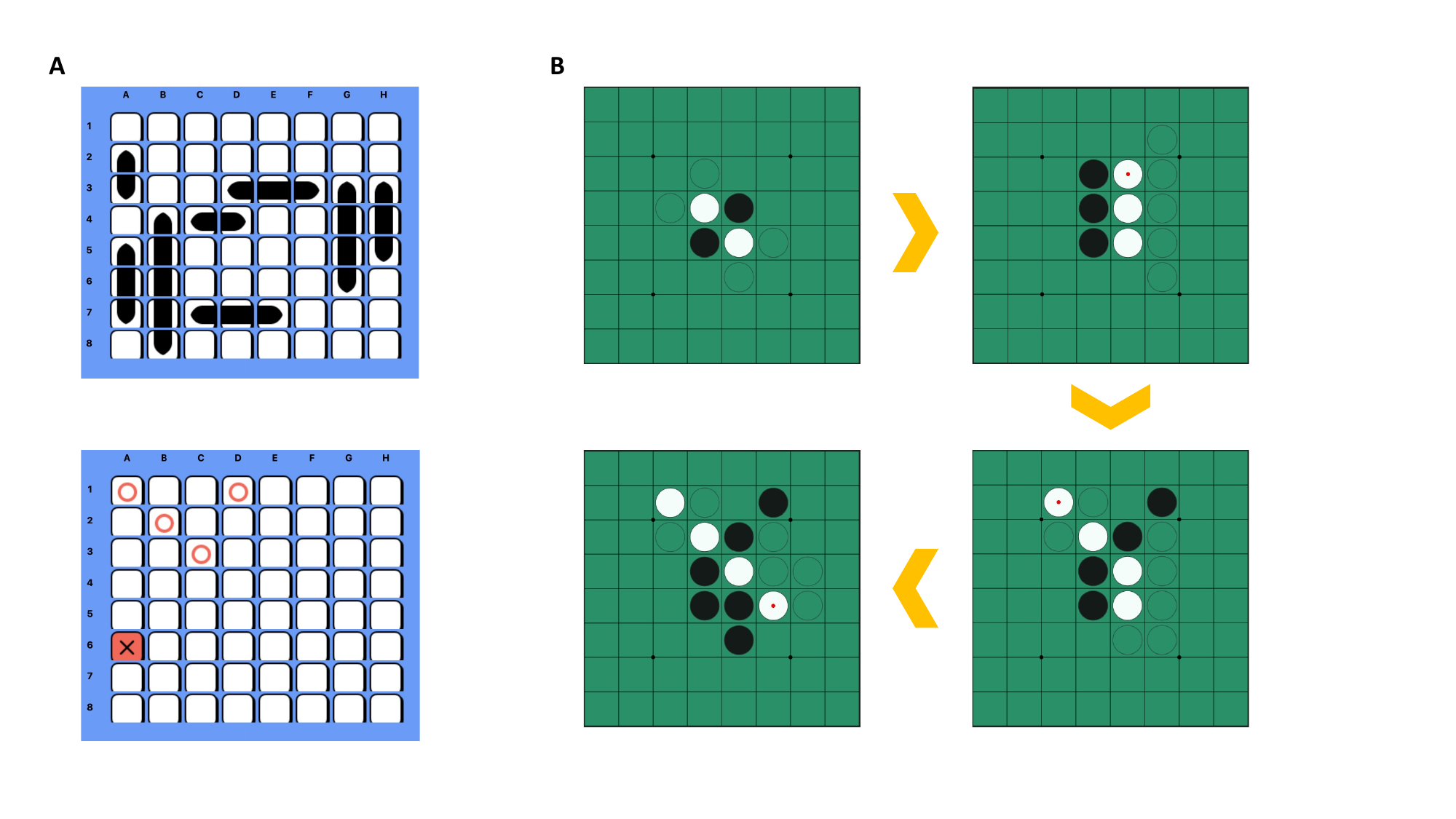}
    \caption{Examples of Battleship and Othello Games.}
    \label{fig:intro-strategy}
\end{figure}

Othello, also known as Reversi, is a two-player deterministic board game that combines simple operational rules with substantial strategic depth. The game is played on a square board of fixed size, traditionally an $8 \times 8$ grid, and uses a set of identical discs that are colored black on one side and white on the other. Each disc represents ownership by one of the two players, and ownership may change during the course of the game through a well-defined flipping mechanism. The fundamental objective of the game is to finish with a greater number of discs of one's own color on the board than the opponent.

At the beginning of the game, the board is initialized with four discs placed at its center(Figure \ref{fig:intro-strategy}B). Two black discs and two white discs occupy these central squares in a diagonal configuration. All remaining squares on the board are empty. One player is assigned the black discs and always moves first, while the other player is assigned the white discs and moves second. Gameplay proceeds in alternating turns. On each turn, a player places a disc of their own color onto an empty square of the board, subject to strict legality conditions. A move is legal if and only if the placement of the disc results in at least one of the opponent’s discs being captured (the dashed circles in Figure \ref{fig:intro-strategy}  are valid movement for the black player.). Capture occurs when a contiguous line of one or more opponent discs becomes sandwiched between the newly placed disc and another disc of the current player’s color. This sandwiching must occur along a straight line in one of the eight possible directions defined by the board geometry, namely horizontal, vertical, or diagonal directions. If no such line exists, the move is illegal and cannot be played. 

If a player has no legal moves available on their turn, they are required to pass. Passing does not involve placing a disc or altering the board state, and the turn is immediately transferred to the opponent. The game continues as long as at least one player has a legal move available. The game terminates when the board is completely filled with discs or when both players are unable to make a legal move consecutively. These termination conditions ensure that the game always reaches a finite conclusion. At the end of the game, scoring is performed by counting the number of discs of each color present on the board. Each disc contributes one point to the owning player’s score. The player with the higher score is declared the winner. In the case where both players have the same number of discs, the game ends in a draw. Notably, victory is determined solely by the final configuration of the board rather than intermediate advantages accrued during play.

In the Othello game setting, we design two types of opponents, denoted as the ``easy'' and ``hard'' agents, each implementing distinct strategies with different levels of tactical sophistication.  The LLM is tasked with developing new strategies that can consistently outperform these opponents. This setup encourages the LLM not only to discover effective local tactics but also to synthesize higher-level strategic patterns that generalize across different styles of adversaries.

\subsection{Battleship}
\label{app:battleship}
This game presents the agent with a two-dimensional grid map on which a fixed fleet of ships is secretly deployed~(Figure \ref{fig:intro-strategy} A). The agent has no prior knowledge of the ships’ exact positions, their shapes, or how many squares each ship occupies. On every turn the agent selects a single cell to “shoot.” Immediately after the shot is fired the environment returns a reward signal: a positive value if a ship was hit and zero otherwise. No other information—such as the identity of the struck ship, its orientation, or how many of its remaining segments are still afloat—is revealed.

The agent’s task is to infer, shot by shot, the most probable locations of the hidden vessels and to maximize the total number of hits (or minimize the number of shots required to sink every ship). Because the only feedback is the sparse binary reward, the agent must maintain an internal belief state—typically a probability distribution over all possible ship configurations—and update that belief using the history of past shots and their outcomes. Efficient exploration, information gain, and probabilistic reasoning are therefore central to strong performance. Performance is evaluated as the ratio of successful hits to the total number of shots taken.